\documentclass[11pt]{article}
\usepackage{mathrsfs}
\usepackage{amsmath}
\usepackage{amsfonts}
\usepackage{amssymb}
\usepackage{latexsym}
\usepackage[dvips]{graphicx}
\usepackage{hyperref}
\newtheorem{theorem}{\bf Theorem}[section]
\newtheorem{lemma}[theorem]{\bf Lemma}

\newtheorem{corollary}[theorem]{\bf Corollary}

\newenvironment{proof}{\noindent{\em Proof:}}{\quad \hfill$\Box$\vspace{2ex}}

\newtheorem{definition}[theorem]{\bf Definition}

\def \bI {{\bf I}}
\def \bN {\Bbb N}

\def \bR {\Bbb R}

\def \bJ {\Bbb J}
\def \bW {{\bf W}}
\def \bW {{\bf W}}
\def \barW {\bar{\bf W}}

\def \barJ {\bar{\bf J}}
\def \barb {\bar{\bf b}}
\def \bbb {{\bf b}}
\def \bw {{\bf \bar{w}}}

\def \bP {{\bf P}}

\def \bx {{\bf x}}
\def \bw {{\bf w}}
\def \by {{\bf y}}

\def \bc {{\bf c}}
\def \bd {{\bf d}}

\def \cB {{\cal B}}
\def \cD {{\cal D}}

\def \cN {{\cal N}}

\def \and {\, \mbox{\rm and}\, }

\def \supp {\,{\rm supp}\,}

\def \diag {\,{\rm diag}\,}

\def \ess {{\rm ess}}
\def \esssup {{\rm ess\ sup}}

\makeatletter

\newcommand{\Rmnum}[1]{\expandafter\@slowromancap\romannumeral #1@}
\makeatother
\begin{document}
\title{\bf Convergence of Deep Convolutional Neural Networks}
\author{Yuesheng Xu\thanks{Department of Mathematics \& Statistics, Old Dominion University, Norfolk, VA 23529, USA. E-mail address: {\it y1xu@odu.edu}. Supported in part by US National Science Foundation under grant DMS-1912958 and by Natural Science Foundation of China under grant 11771464. }
\quad and \quad Haizhang Zhang\thanks{School of Mathematics (Zhuhai), Sun Yat-sen University, Zhuhai, P.R. China. E-mail
address: {\it zhhaizh2@sysu.edu.cn}. Supported in part by National Natural Science Foundation of China under grant 11971490, and by Natural Science Foundation of Guangdong Province under grant 2018A030313841. Corresponding author.}}
\date{}
\maketitle

%
%

\begin{abstract}
Convergence of deep neural networks as the depth of the networks tends to infinity is fundamental in building the mathematical foundation for deep learning. In a previous study, we investigated this question for deep ReLU networks with a fixed width. This does not cover the important convolutional neural networks where the widths are increasing from layer to layer. For this reason, we first study convergence of general ReLU networks with increasing widths and then apply the results obtained to deep convolutional neural networks. It turns out the convergence reduces to convergence of infinite products of matrices with increasing sizes, which has not been considered in the literature. We establish sufficient conditions for convergence of such infinite products of matrices. Based on the conditions, we present sufficient conditions for piecewise convergence of general deep ReLU networks with increasing widths, and as well as pointwise convergence of deep ReLU convolutional neural networks.

\medskip


\end{abstract}

\section{Introduction}

Deep neural networks make great contributions to recent successes of machine learning in various applications such as image recognition, speech recognition, game intelligence, natural language processing, and autonomous navigation, \cite{Goodfellow,LeCun}. The remarkable successes have stimulated many mathematical researches in deep neural networks, with the hope of building a rigorous mathematical foundation to increase the interpretability of deep learning models, \cite{Adcock,Daubechies,Devore1,E,Elbrachter,Montanelli1,Montanelli2,Poggio,Shen1,Shen2,Shen3,Wang,Yarotsky,Zhou1}. In particular, the approximation and expressive powers of deep neural networks have been a focus. A brief review of the literature for this aspect was contained in the introduction of \cite{XuZhang}. We also refer readers to the two recent surveys \cite{Devore1,Elbrachter} for further details.

Different from those studies on approximation and expressive powers, we are concerned with the convergence of deep neural networks (DNNs) as the depth is increasing to infinity. In Fourier analysis, the convergence of a Fourier series in terms of its coefficients is a fundamental question, \cite{Zygmund}. From this perspective, we desire to characterize convergence of deep ReLU networks as the depth is increasing to infinity in terms of the weight matrices and bias vectors of the networks. We believe that this question is fundamental in mathematical theory of deep learning. In the recent work \cite{XuZhang}, we studied this question for deep ReLU networks with a fixed width and were able to provide a mathematical interpretation to the design of the successful deep residual networks \cite{KaimingHe} based on the obtained results.

However, the study \cite{XuZhang} and those studies \cite{Adcock,Daubechies,Devore1,E,Elbrachter,Montanelli1,Montanelli2,Poggio,Shen1,Shen2,Shen3,Yarotsky} on approximation and expressive powers of deep neural networks were all about fully connected networks without pre-specified structures, and thus did not cover the important convolutional neural networks (CNNs). CNNs use filter masks to perform convolution with the input data at each layer and are particularly useful and efficient in extracting features of speech and image data. Moreover, the weight matrices in CNN are sparse Toeplitz matrices compared to dense weight matrices in a fully connected deep neural network. This endows a sparse nature to CNNs. These ingredients account for the dominance of CNNs in machine learning. In fact, most of the powerful deep learning architectures, such as AlexNet, VGG, GoogleNet, and ResNet, are different types of CNNs, \cite{Goodfellow,LeCun}. We review some theoretical studies that are applicable to CNNs. Universality of CNNs was first established in \cite{Zhou1}. Several versions of the
universal approximation theorem for CNNs were recently proved in \cite{Yarotsky2}. Equivalence  of approximation by CNNs and fully-connected neural networks was explored in \cite{Petersen}. Lipschitz constants and proximal properties of neural networks were investigated in \cite{Combettes} and \cite{Hasannasab}, respectively.

We shall consider convergence of neural networks with a pure convolutional structure. Thus, similar to the treatment in \cite{Zhou1}, other engineering techniques such as pooling or batch normalization will not be considered. For such CNNs, the widths are not fixed but are increasing to infinity. This was not considered in \cite{XuZhang} where the width of neural networks is fixed. There are some major difficulties caused by the increasing widths. First, convergence of the functions determined by such CNNs need to be carefully defined. Second, we encounter an issue about infinite products of matrices with increasing sizes, which has not been considered in the field of matrix analysis to our best knowledge. Our first strategy is to extend the functions determined by CNNs to have infinite components by adding zeros to their tails and study the convergence in a space of vector-valued functions of infinite dimension. A second key strategy is to formulate the problem about infinite products of matrices with increasing sizes as a problem on the convergence of linear operators. In these manners, we are able to derive sufficient conditions for convergence of such infinite products of matrices. Based on the conditions, we present sufficient conditions for pointwise convergence of deep convolutional neural networks.


The rest of this paper is organized as follows. In Section 2, we describe the matrix form of CNNs, which expresses CNNs as a special form of DNNs.
In Section 3, we review the definition and notation of general neural networks with increasing widths and define the notion of convergence of neural networks when new layers are paved to the existing network so that the depth is increasing to infinity. Then, by introducing the notions of activation domains and activation matrices, we derive an explicit expression of deep ReLU networks. With the expression, we connect convergence of deep ReLU networks with increasing widths with the existence of two limits involving infinite products of matrices with increasing sizes. Sufficient conditions for convergence of such infinite products of matrices are established in Section 4. Based on these results, we obtain in Section 5 sufficient conditions for the pointwise convergence of deep ReLU networks with increasing widths. Finally, in Section 6 we apply the general results to establish convenient sufficient conditions for convergence of ReLU CNNs.

\section{Convolutional Neural Networks}\label{CNNs}
\setcounter{equation}{0}

For the purpose of understanding convergence of convolutional neural networks (CNNs), we consider in this section the matrix form of CNNs. Motivated by CNNs, we describe a framework of deep neural networks with increasing widths.

We review the pure convolutional structure of a ReLU neural network.
Let $\sigma$ denote the {\bf ReLU} activation function
$$
\sigma(x):=\max(x,0),\ \ x\in\bR.
$$
A neural network with the pure convolutional structure and the ReLU activation function may be illustrated as follows:
\begin{equation}\label{cnn1}
\begin{aligned}
 x\in[0,1]^d\ &  \xrightarrow[\sigma]{\bw^{(1)},\bbb_1} \  x^{(1)} &\xrightarrow[\sigma]{\bw^{(2)},\bbb_2} \  x^{(2)} &\rightarrow\cdots\rightarrow&\xrightarrow[\sigma]{\bw^{(n)},\bbb_n}  x^{(n)}. \\
\mbox{input}\quad& \quad\mbox{1st layer}&\mbox{ 2nd layer} & & \mbox{$n$-th layer}
\end{aligned}
\end{equation}
In (\ref{cnn1}), $\bw^{(i)}:=(\bw^{(i)}_0,\bw^{(i)}_1,\dots,\bw^{(i)}_{s_i})\in\bR^{s_i+1}$ and $\bbb_i$ are the filter mask and the bias vector at the $i$-th layer, respectively, and
$$
x^{(i)}:=\sigma(x^{(i-1)}*\bw^{(i)}+\bbb_i),\ \ 1\le i\le n,\ \mbox{ with }x^{(0)}:=x.
$$
Here the convolution $\bx*\bw$ of a vector $\bx:=(\bx_1, \bx_2,\dots,\bx_m)\in\bR^m$ with a filter mask $\bw:=(\bw_0,\bw_1,\dots,\bw_s)$ outputs a vector $\by:=(\by_1, \by_2,\dots, \by_{m+s})\in\bR^{m+s}$ defined by
$$
\by_i:=\sum_{j=\max(0,i-m)}^{\min(i-1,s)}\bw_j\bx_{i-j},\ \ 1\le i\le m+s.
$$
Note that we do not consider the output layer as it is simply a linear transformation and does not affect the convergence of neural networks.

The goal of this paper is to understand under what conditions on the filter mask and the bias vector, the convolutional neural networks converge to a meaningful function of the input vector. For this purpose, we find it convenient to express the convolution $\bx*\bw$ via multiplication of $x$ with a matrix. To this end, we define an $(m+s)\times m$ Toeplitz
type matrix by
$$
T_{ij}:=\left\{
\begin{aligned}
0,&\quad i<j,\\
\bw_{i-j},&\quad i\ge j,
\end{aligned}
\right.
$$
and rewrite the convolution operation in the form of vector-matrix multiplication in terms of the matrix $T$ as
$$
\by=\bx*\bw=T\bx.
$$
Clearly, Matrix $T$ can be expressed by
$$
T=\left[
\begin{array}{ccccccc}
     \bw_0&0&0&0&\cdots &\cdots &0 \\
     \bw_1&\bw_0&0&0&\cdots&\cdots &0\\
     \vdots&\ddots&\ddots&\ddots&\ddots&\ddots&\vdots\\
     \bw_s&\bw_{s-1}&\cdots&\bw_0&0&\cdots&0\\
     0&\bw_s&\cdots&\bw_1&\bw_0&0\cdots&0\\
     \vdots&\ddots&\ddots&\ddots&\ddots&\ddots&\vdots\\
     \cdots&\cdots&0&\bw_s&\bw_{s-1}&\cdots&\bw_0\\
     \cdots&\cdots&\cdots&0&\bw_s\cdots&\cdots&\bw_1\\
     \vdots&\ddots&\ddots&\ddots&\ddots&\ddots&\vdots\\
     0&\cdots&\cdots&\cdots0& \cdots&\bw_s&\bw_{s-1}\\
     0&\cdots&\cdots&\cdots&\cdots&\cdots0&\bw_s
\end{array}
\right].
$$
Note that entries of $T$ along each of the diagonal and sub-diagonals are constant.
By adopting the above matrix form of the convolution, the CNN (\ref{cnn1}) is a special case of the deep neural network with increasing widths. Specifically, we let
\begin{equation}\label{cnnwidth}
    m_0:=d, \ \ m_n:=m_{n-1}+s_n,\ \ n\in\bN
\end{equation}
and $\bW_n\in\bR^{m_n\times m_{n-1}}$ be defined by
\begin{equation}\label{cnnweightmatrix}
(\bW_n)_{jk}:=\left\{
\begin{aligned}
0,&\quad j<k\\
\bw^{(n)}_{j-k},&\quad j\ge k,
\end{aligned}
\quad 1\le j\le m_n,1\le k\le m_{n-1}.
\right.
\end{equation}
Then
$$
x^{(i)}=\sigma(\bW_ix^{(i-1)}+\bbb_i),\ \ 1\le i\le n,
$$
which is a ReLU deep neural network with increasing widths (with weight matrices $\bW_i$ of increasing sizes).

Having the above matrix form, conditions on the masks and bias vectors that ensure convergence of the convolutional neural network (\ref{cnn1}) as the depth $n$ tends to infinity are now reformulated as conditions on the weight matrices and the bias vectors. Convergence of deep ReLU neural networks with a fixed width was recently studied in \cite{XuZhang}. The matrix form for CNNs differs from that for DNNs considered in \cite{XuZhang} in that the widths in CNNs are increasing while those in \cite{XuZhang} are fixed. This difference causes major difficulties that will be clear as we proceed with the analysis. As a consequence, the results in \cite{XuZhang} are not applicable directly to the current setting. Nevertheless, the idea of replacing the application of the ReLU activation function by multiplication with certain activation matrices introduced in \cite{XuZhang} will be adopted.
Since CNN (\ref{cnn1}) is a special deep ReLU neural network with increasing widths, we shall first study convergence of general fully connected feed-forward neural networks with non-decreasing widths. The results obtained will then be applied to establishment of convergence of CNNs.

\section{Deep Neural Networks with Increasing Widths}
\setcounter{equation}{0}
In this section, we describe the setting of the general fully connected neural networks with non-decreasing widths and increasing depth and formulate their convergence as convergence of infinite products of non-square matrices.

We first describe the setting of general deep ReLU neural networks with non-decreasing widths.
Let $\{m_i:i\in\bN\}$ be a sequence of widths that are non-decreasing:
$$
m_i\le m_{i+1},\ \ i\in\bN.
$$
We consider deep ReLU networks from input domain $[0,1]^d\subseteq\bR^d$ to the output space $\bR^{d'}$. For each $1\le i\le n$, let $\bW_i$ and $\bbb_i$ denote respectively the weight matrix and the bias vector of the $i$-th hidden layer. That is, $\bbb_i\in\bR^{m_i}$ for $1\le i\le n$, and $\bW_i\in\bR^{m_i\times m_{i-1}}$ for $1\le i\le n$, where $m_0:=d$. The weight matrix $\bW_{o}$ and the bias vector $\bbb_{o}$ of the output layer satisfy $\bW_{o}\in \bR^{d'\times m_n}$ and $\bbb_{o}\in\bR^{d'}$. The structure of such a deep neural network
with the ReLU activation function $\sigma$
is illustrated as follows:
\begin{equation}\label{neuralnetworks}
\begin{aligned}
 x\in[0,1]^d\ &  \xrightarrow[\sigma]{\bW_1,\bbb_1} \  x^{(1)} &\xrightarrow[\sigma]{\bW_2,\bbb_2} \  x^{(2)} &\rightarrow\cdots\rightarrow&\xrightarrow[\sigma]{\bW_n,\bbb_n}  x^{(n)}&\xrightarrow{\bW_{o},\bbb_{o}}\  y\in\bR^{d'}. \\
\mbox{input}\quad& \quad\mbox{1st layer}&\mbox{ 2nd layer} & & \mbox{$n$-th layer}&\quad\mbox{ output}
\end{aligned}
\end{equation}
Here
\begin{equation}\label{neuralnetworks-termk}
x^{(k)}:=\sigma(\bW_{k}x^{(k-1)}+\bbb_k),\ \ 1\le k\le n \ \mbox{ with } \ x^{(0)}=x,
\end{equation}
\begin{equation}\label{neuralnetworks-output}
y:=\bW_{o}x^{(n)}+b_{o},
\end{equation}
and the activation function $\sigma$ is applied to a vector componentwise. Thus, the above deep neural network determines a continuous function $x\to y$ from $[0,1]^d$ to $\bR^{d'}$.  Below, we recall the compact notation for consecutive compositions of functions which was used in \cite{XuZhang}.

\begin{definition} {\bf (Consecutive composition)}
\label{Consecutive-compostion}
Let $f_1, f_2, \dots,f_n$ be a finite sequence of functions such that the range of $f_i$ is contained in the domain of $f_{i+1}$, $1\le i\le n-1$, the consecutive composition of $\{f_i\}_{i=1}^n$ is defined to be function
$$
\bigodot_{i=1}^n f_i:=f_n\circ f_{n-1}\circ\cdots\circ f_2\circ f_1,
$$
whose domain is that of $f_1$.
\end{definition}

Using the above notation, equations \eqref{neuralnetworks-termk} and \eqref{neuralnetworks-output} may be rewritten as
$$
x^{(k)}=\left(\bigodot_{i=1}^k \sigma(\bW_i \cdot+\bbb_i)\right)(x),\ \ 1\le k\le n
$$
and
$$
y=\bW_{o}\left(\bigodot_{i=1}^n \sigma(\bW_i \cdot+\bbb_i)\right)(x)+b_{o},\ \ x\in[0,1]^d,
$$
respectively.
We are concerned with convergence of the above functions determined by the deep neural network as $n$ increases to infinity. One sees that the output layer is a linear function of $x^{(n)}$ and thus, it will not be considered in the convergence. We are hence concerned with the convergence of the deep neural network defined by
$$
\cN_n(x):=\left(\bigodot_{i=1}^n \sigma(\bW_i \cdot+\bbb_i)\right)(x),\ \ x\in[0,1]^d
$$
as $n$ tends to infinity. Note that $\cN_n$ is a function from $[0,1]^d$ to $\bR^{m_n}$ and $m_n$ is non-decreasing on $n$, unlike the scenario considered in \cite{XuZhang} where the matrix size and vector size were fixed. While in the present case, DNNs $\cN_n$ are vector-valued functions in $\bR^{m_n}$. In other words, the output dimension of $\cN_n$ is not fixed but grows as $n$ increases.
This brings difficulties for investigating the function sequence $\cN_n$ as $n$ changes.
To overcome the difficulty, we shall extend each $\cN_n$ as a vector in the infinite dimensional space $\ell^p$ for $1\le p\le+\infty$, the space of infinite sequences that are $\ell^p$ summable, by adding zeros to its tail, and consider convergence of $\cN_n$ as a vector-valued function in an appropriate topology related to $\ell^p$. Specifically, for $x\in [0,1]^d$, we define
$$
\tilde{\cN}_n(x)_j:=\cN_n(x)_j, \mbox{ for } 1\le j\le m_n\mbox{ and }\tilde{\cN}_n(x)_j:=0, \mbox{ for }j>m_n.
$$
In this way, convergence of $\cN_n$ in a finite dimensional space $\bR^{m_n}$ is recast as convergence of $\tilde{\cN}_n$ in the infinite space $\ell^p$.

Convergence of $\tilde{\cN}_n$ should be considered in an appropriate topology related to  $\ell^p$. We now specify it.
By $\|\cdot\|_p$ we denote the $\ell^p$-norm, $1\le p\le+\infty$. For a Lebesgue measurable subset $\Omega\subseteq\bR^d$, by $L^q(\Omega,\ell^p)$ we denote the space of all real-valued functions $f:\Omega\to\ell^p$ such that each component of $f$ is Lebesgue measurable on $\Omega$ and such that
$$
\|f\|_{L^q(\Omega,\ell^p)}:=\left\{
\begin{array}{ll}
\displaystyle{\biggl(\int_{\Omega}\|f(x)\|_p^q dx\biggr)^{1/q}},&1\le  q<+\infty,\\
\displaystyle{\ess\sup_{x\in\Omega} \|f(x)\|_p},&q=+\infty
\end{array}
\right.
$$
is finite. Note that $L^q(\Omega,\ell^p)$ is a Banach space.

\begin{definition} {\bf (Convergence of neural networks)} Let $\bW:=\{\bW_n\}_{n=1}^\infty$ with $\bW_n\in\bR^{m_n\times m_{n-1}}$, $n\in\bN$ be a sequence of weight matrices, and $\bbb:=\{\bbb_n\}_{n=1}^\infty$ with $\bbb_n\in\bR^{m_n}$ be a sequence of bias vectors. We say the deep ReLU network $\cN_n$ determined by $\bW$, $\bbb$ converges to a limit function $\cN$ in $L^q([0,1]^d,\ell^p)$  if
$$
\lim_{n\to\infty} \|\tilde{\cN}_n-\cN\|_{L^q([0,1]^d,\ell^p)}=0.
$$
We say $\cN_n$ converges pointwise to some function $\cN$ if for each $x\in[0,1]^d$,
$$
\lim_{n\to\infty} \|\tilde{\cN}_n(x)-\cN(x)\|_{p}=0.
$$
\end{definition}

We shall study conditions on the weight matrices and the bias vectors that ensure convergence of the deep ReLU network. These results will extend those established in \cite{XuZhang} for the case with fixed matrix sizes. This extension is nontrivial. As explained earlier, the matrix size and vector size considered in \cite{XuZhang} were fixed, and as a result, DNNs for this case were treated as finite-dimensional vector-valued functions. Now in the current situation, due to the growth of the matrix size, DNNs are vector-valued functions in $\ell^p$.

We introduce an algebraic formulation of a deep ReLU network by adapting the notions of activation domains and activation matrices introduced in \cite{XuZhang}.
For each $m\in\bN$, we define the set of {\bf activation matrices} by
$$
\cD_m:=\left\{\diag(a_1,a_2,\dots,a_m):a_i\in\{0,1\},1\le i\le m\right\}.
$$
The support of an activation matrix $J\in \cD_m$ is defined by
$$
\supp J:=\{k:\ J_{kk}=1,\ \ 1\le k\le m\}.
$$
The set $\cD_m$ of the activation matrices has exactly $2^m$ elements. This matches the number of possible different activation patterns of a ReLU neural network.
Therefore, it is convenient to use $\cD_m$ as an index set. We need the notion of activation domains of one layer network, which was originally introduced in \cite{XuZhang}.

\begin{definition} {\bf (Activation domains of one layer network) } \label{ActDomofOne}
For a weight matrix $\bW\in \bR^{m\times m'}$ and a bias vector $\bbb\in\bR^m$, the activation domain of $\sigma(\bW x+\bbb)$ with respect to a diagonal matrix $J\in\cD_m$ is
$$
D_{J,\bW,\bbb}:=\left\{x\in\bR^{m'}: (\bW x+\bbb)_j>0\ \mbox{ for }j\in\supp J\mbox{ and }(\bW x+\bbb)_j\le0\mbox{ for }j\notin\supp J\right\}.
$$
\end{definition}

Definition \ref{ActDomofOne} enables us to construct a partition of the unit cube $[0,1]^d$ that corresponds to the piecewise linear component of the one-layer neural network function
$$
\cN_1(x):=\sigma(\bW_1x+\bbb_1),\ \ x\in[0,1]^d.
$$
Specifically, we have
\begin{equation}\label{partition}
[0,1]^d=\bigcup_{J_1\in\cD_{m_1}}(D_{J_1,\bW_1, \bbb_1}\cap [0,1]^d)
\end{equation}
and
\begin{equation}\label{expressionofN1}
\cN_1(x)=J_1(\bW_1x+\bbb_1), \ \ x\in D_{J_1,\bW_1, \bbb_1}\cap [0,1]^d, \ \ \mbox{for}\ \ J_1\in \cD_{m_1}.
\end{equation}
Note that in equation \eqref{partition}, the set $\cD_{m_1}$ is used as an index set.
The essence of equation \eqref{expressionofN1} is that the application of the ReLU activation function in the definition of the one-layer neural network  $\cN_1$ is replaced by multiplication with an activation matrix in $\cD_{m_1}$. We refer readers to \cite{XuZhang} for more explanation of activation domains.

We next introduce the activation domain of a multi-layer network with non-square weight matrices, which extends the activation domain of a multi-layer network with square matrices originated in  \cite{XuZhang}.

\begin{definition}\label{multiactivationdomain} {\bf (Activation domains of a multi-layer network)}
For
$$
\barW_n:=(\bW_1,\dots,\bW_n)\in\prod_{i=1}^n\bR^{m_i\times m_{i-1}},\ \ \mbox{and}\ \ \barb_n:=(\bbb_1,\dots,\bbb_n)\in\prod_{i=1}^n \bR^{m_i},
$$
the activation domain of
$$
\bigodot_{i=1}^n\sigma(\bW_i\cdot+\bbb_i)
$$
with respect to $\barJ_n:=(J_1,\dots,J_n)\in \prod_{i=1}^n\cD_{m_i}$ is defined recursively by
$$
D_{\barJ_1,\barW_1,\barb_1}= D_{J_1,\bW_1,\bbb_1}\cap[0,1]^d
$$
and
$$
D_{\barJ_n,\barW_n,\barb_n}=\left\{x\in D_{\barJ_{n-1},\barW_{n-1},\barb_{n-1}}:\ \biggl(\bigodot_{i=1}^{n-1}\sigma(\bW_i\cdot+\bbb_i)\biggr)(x)\in D_{J_n,\bW_n,\bbb_n}\right\}.
$$
\end{definition}

For each positive integer $n$, the activation domains
$$
D_{\barJ_n,\barW_n,\barb_n},\ \  \mbox{for}\ \ \barJ_n:=(J_1,\dots,J_n)\in \prod_{i=1}^n\cD_{m_i},
$$
form a partition of the unit cube $[0,1]^d$. By using these activation domains, we are able to write down an explicit expression of the ReLU network $\cN_n$ with applications of the ReLU activation function replaced by multiplications with the activation matrices.
To this end, we write
$$
\prod_{i=1}^n\bW_i=\bW_n\bW_{n-1}\cdots\bW_1.
$$
For $n,k\in \bN$, we also adopt the following convention that
$$
\prod_{i=k}^n\bW_i=\bW_n\bW_{n-1}\cdots\bW_k, \ \ \mbox{for} \ \  n\geq k, \ \ \mbox{and}\ \
\prod_{i=k}^n\bW_i= \bI_{m_{n}}, \ \ \mbox{for} \ \  n< k,
$$
where for each $m$, ${ \bI_m}$ denotes the $m\times m$ identity matrix.

\begin{theorem}\label{expressionofcnprop}
It holds that
\begin{equation}\label{compositions2}
\cN_n(x)=\biggl(\prod_{i=1}^nJ_i\bW_i\biggr)x+\sum_{i=1}^n\biggl(\prod_{k=i+1}^n J_k\bW_k\biggr)J_i\bbb_i,\ \ x\in D_{\barJ_n,\barW_n,\barb_n},  \ \ \barJ_n:=(J_1,\dots,J_n)\in \prod_{i=1}^n \cD_{m_i}.
\end{equation}
\end{theorem}
\begin{proof}
The proof is similar to that of Theorem 3.4 in \cite{XuZhang}.
\end{proof}

Theorem \ref{expressionofcnprop} differs from  Theorem 3.4 in \cite{XuZhang} in sizes of the matrices involved in the representation of the network.
The representation of a deep ReLU network established in Theorem \ref{expressionofcnprop} is helpful in reformulating convergence of deep ReLU networks as convergence of infinite products of matrices.
For $m'\ge m$, we define the $m'\times m$ real matrix
$$
\bI_{m',m}:=\left(\begin{array}{c}\bI_m\\{\bf 0}\end{array}\right),
$$
where ${\bf 0}$ denotes the $(m'-m)\times m$ zero matrix. Moreover, we use $\bI_{\infty,m}$ for the matrix with infinite rows and $m$ columns whose top $m\times m$ submatrix equals to $\bI_m$ and whose remaining entries are all zero. With the notation, one sees that
\begin{equation}\label{extension2}
\tilde{\cN}_n(x)=\bI_{\infty,m_n}\cN_n(x),\ \ x\in[0,1]^d.
\end{equation}
In other words, the extension of $\cN_n$ to $\tilde{\cN}_n$ is accomplished by multiplication of $\cN_n$ with $\bI_{\infty,m_n}$ from the left.

Convergence of the deep ReLU neural networks $\cN_n$ may be characterized in terms of the weight matrices and bias vectors.
Let $\bW:=\{\bW_n\}_{n=1}^\infty$ with $\bW_n\in \bR^{m_n\times m_{n-1}}$ be a sequence of weight matrices and $\bbb:=\{\bbb_n\}_{n=1}^\infty$ with $\bbb_n\in\bR^{m_n}$ be a sequence of bias vectors. Suppose that $\cN\in L^q([0,1]^d,\ell^p)$.
It follows from Theorem \ref{expressionofcnprop} that the ReLU neural networks $\cN_n$ converge to $\cN$ in $L^q([0,1]^d,\ell^p)$ if and only if for $1\leq q<+\infty$
$$
    \lim_{n\to\infty}\sum_{\barJ_n\in \prod_{i=1}^n\cD_{m_i}}\int_{D_{\barJ_n,\barW_n,\barb_n}}\biggl\|
    \bI_{\infty,m_n}\biggl(\prod_{i=1}^nJ_i\bW_i\biggr)x+ \bI_{\infty,m_n}\sum_{i=1}^n\biggl(\prod_{k=i+1}^n J_k\bW_k\biggr)J_i\bbb_i-\cN(x)
    \biggr\|_p^qdx=0,
$$
and for $q=\infty$
$$
    \lim_{n\to\infty}\max_{\barJ_n\in \prod_{i=1}^n\cD_{m_i}}\esssup_{x\in D_{\barJ_n,\barW_n,\barb_n}}\biggl\|
    \bI_{\infty,m_n}\biggl(\prod_{i=1}^nJ_i\bW_i\biggr)x+ \bI_{\infty,m_n}\sum_{i=1}^n\biggl(\prod_{k=i+1}^n J_k\bW_k\biggr)J_i\bbb_i-\cN(x)
    \biggr\|_{p}=0.
$$

This necessary and sufficient condition leads to a useful sufficient condition for pointwise convergence of deep ReLU neural networks with increasing widths.

The space $\cB(\bR^m,\bR^{m'})$ of bounded linear operators from $\bR^m$ to $\bR^{m'}$ coincides with the space of all $m'\times m$ real matrices with the induced matrix norm
$$
\|A\|_p:=\sup_{x\in\bR^m,x\ne0}\frac{\|Ax\|_p}{\|x\|_p},\ \ \mbox{for} \ \ A\in \bR^{m'\times m}.
$$
In the right hand side of the above equation, $\|\cdot\|_p$ denotes the $\ell^p$-norm restricted to $\bR^m$ or $\bR^{m'}$.
The space $\cB(\bR^m,\ell^p)$ consists of all bounded linear operators from $\bR^m$ to $\ell^p$ and each operator in $\cB(\bR^m,\ell^p)$ may be expressed as a matrix with infinite rows and $m$ columns. For each $A\in \cB(\bR^m,\ell^p)$, its operator norm, denoted by $\|A\|_p$, is given by
$$
\|A\|_p:=\sup_{x\in\bR^m,x\ne0}\frac{\|Ax\|_p}{\|x\|_p},\ \ \mbox{for} \ \  A\in \cB(\bR^m,\ell^p).
$$
Note that the embedding of $\bR^m$ to $\ell^p$ corresponds to $\bI_{\infty,m}$. Also, both $\cB(\bR^m,\bR^{m'})$ and $\cB(\bR^m,\ell^p)$ are Banach spaces, \cite{Lax}.

\begin{theorem}\label{convergenceRELU}
Let $\bW:=\{\bW_n\}_{n=1}^\infty$ with $\bW_n\in \bR^{m_n\times m_{n-1}}$ be a sequence of weight matrices and $\bbb:=\{\bbb_n\}_{n=1}^\infty$ with $\bbb_n\in\bR^{m_n}$ be a sequence of bias vectors. If for all sequences of diagonal matrices $\bJ=(J_n\in\cD_{m_n}:\ n\in\bN)$, the two limits
\begin{equation}\label{limit1}
\prod_{i=1}^\infty J_i\bW_i:=\lim_{n\to\infty}\bI_{\infty,m_n}\prod_{i=1}^nJ_i\bW_i\quad\mbox{ in }\cB(\bR^d,\ell^p)
\end{equation}
and
\begin{equation}\label{limit2}
\lim_{n\to\infty}\bI_{\infty,m_n}\sum_{i=1}^n\left(\prod_{k=i+1}^n J_k\bW_k\right)J_i\bbb_i \quad\mbox{ in }\ell^p
\end{equation}
both exist, then the sequence of ReLU neural networks $\{\cN_n\}_{n=1}^\infty$ converges pointwise on $[0,1]^d$.
\end{theorem}

\begin{proof}
For every $x\in[0,1]^d$, there exists a sequence of matrices $\bJ=(J_n\in\cD_{m_n}:\ n\in\bN)$ such that
$$
x\in\bigcap_{n=1}^\infty D_{\barJ_n,\barW_n,\barb_n}.
$$
Thus, by (\ref{compositions2}) and (\ref{extension2}), we have
$$
\tilde{\cN}_n(x)=\bI_{\infty,m_n}\biggl(\prod_{i=1}^nJ_i\bW_i\biggr)x+\bI_{\infty,m_n}\sum_{i=1}^n\biggl(\prod_{k=i+1}^n J_k\bW_k\biggr)J_i\bbb_i,\ \ n\in\bN.
$$
Therefore, the existence of the two limits (\ref{limit1}) and (\ref{limit2}) are sufficient for pointwise convergence of $\{\cN_n(x)\}$ by the completeness of $\ell^p$.
\end{proof}

\section{Infinite Products of Matrices with Increasing Sizes}
\setcounter{equation}{0}

By Theorem \ref{convergenceRELU}, existence of the two limits (\ref{limit1}) and (\ref{limit2}) serves as a sufficient condition to ensure pointwise convergence  of deep ReLU neural networks. In particular, convergence of the infinite product of matrices
\begin{equation}\label{infinitematrixproducts}
\prod_{n=1}^\infty J_n\bW_n, \ \ \mbox{for any}\ \ J_n\in \cD_{m_n},
\end{equation}
with increasing sizes, appears in both of the limits. We hence study this important issue in this section. Note that the infinite product of matrices \eqref{infinitematrixproducts} differs from (4.1) of \cite{XuZhang} in the sizes of matrices $\bW_n$. Here, they are non-square matrices, and while in \cite{XuZhang} they are square matrices of a fixed size. Due to this fact, results of \cite{XuZhang} are not directly applicable to \eqref{infinitematrixproducts} here.

Note that for each $n\in\bN$,
$$
\prod_{i=1}^nJ_i\bW_i
$$
is of size $m_n\times d$. Thus, for different $n$, these matrices may have different numbers of rows. Connecting to our main problem of convergence of ReLU networks in Theorem \ref{convergenceRELU}, we shall treat each matrix $\prod_{k=1}^n J_k\bW_k$ as a linear operator in $\cB(\bR^d,\ell^p)$ defined as
$$
\bI_{\infty,m_n}\biggl(\prod_{i=1}^nJ_i\bW_i\biggr)
$$
and consider their convergence in $\cB(\bR^d,\ell^p)$. In conclusion, the limit (\ref{infinitematrixproducts}) is defined as in (\ref{limit1}).

Infinite products of square matrices have been well studied (see, for example, \cite{Artzrouni,Wedderburn}). There is a well-known sufficient condition for such infinite products to converge (\cite{Wedderburn}, page 127). The result states that if matrices $\bW_n\in \bR^{m\times m}$ have the form
\begin{equation}\label{MatrixForm}
\bW_n=\bI_m+\bP_n
\end{equation}
with
\begin{equation}\label{ConditionMatrixPn}
\sum_{n=1}^\infty\|\bP_n\|<+\infty,
\end{equation}
then the infinite product $\prod_{n=1}^\infty \bW_n$ converges.
In condition \eqref{ConditionMatrixPn}, $\|\cdot\|$ denotes a matrix norm having the property that
$$
\|AB\|\le \|A\|\|B\|,
$$
for all matrices $A$ and $B$ whose multiplication $AB$ is well-defined.
Our question differs from the above result in having the diagonal matrices $J_n$ in (\ref{infinitematrixproducts}) arbitrarily chosen from $\cD_{m_n}$, and differs from the result in \cite{XuZhang} in the aspect that the matrices $W_n$ in (\ref{infinitematrixproducts})  have increasing sizes.

To our best knowledge, the convergence issue of infinite products of non-square matrices with varying sizes has not been studied before.

We shall establish a sufficient condition for convergence of infinite products (\ref{infinitematrixproducts}) of non-square matrices with varying sizes. This poses extra difficulty compared to that in \cite{XuZhang}. Here, we have to carefully deal with three classes of matrices $\bI_{m_i,m_{i-1}}, J_i,\bW_i$. To this end, we shall always use $\|\cdot\|_p$ for the matrix norm induced from the $\ell^p$ vector norm. Thus, this matrix norm has the properties that
\begin{equation}\label{matrixcon1}
\|AB\|_p\le \|A\|_p\|B\|_p \mbox{ for all matrices }A,B,
\end{equation}
\begin{equation}\label{matrixcon2}
\|\bI_{m',m}\|_p=1 \mbox{ for all }m'\ge m,
\end{equation}
and
\begin{equation}\label{matrixcon3}
\|J\|_p\le 1 \mbox{ for all }J\in\cD_m,
\end{equation}

We need the following lemma which was proved in \cite{XuZhang}.

\begin{lemma}\label{sufficientlemma}
If a sequence $\{a_n\}_{n=1}^\infty$ satisfies $a_n\geq 0$ and $\sum_{n=1}^\infty a_n<+\infty$, then for all $q\in\bN$,
\begin{equation}\label{sufficientlemma2}
\sum_{i=q+1}^\infty a_i+\sum_{l=2}^\infty\sum_{\substack{1\le i_1<i_2<\cdots<i_l\\ i_l>q}}\prod_{k=1}^la_{i_k}\le \left(\sum_{i=q+1}^\infty a_i\right)\exp\left(\sum_{i=1}^\infty a_i\right).
\end{equation}
Consequently, if $a_i$, $1\le i\le n$, are nonnegative numbers, then for all $q<n$,
\begin{equation}\label{sufficientlemma1}
\sum_{i=q+1}^n a_i+\sum_{l=2}^n\sum_{\substack{1\le i_1<i_2<\cdots<i_l\le n\\ i_l>q}}\prod_{k=1}^la_{i_k}\le \left(\sum_{i=q+1}^na_i\right)\exp\left(\sum_{i=1}^na_i\right).
\end{equation}
\end{lemma}

We need a second observation regarding the product of activation matrices.

\begin{lemma}\label{sufficientlemma}
If $k\in\bN$ and $J_i\in\cD_{m_i}$ for $i\ge k$, then
\begin{equation}\label{limit00}
\lim_{n\to \infty}\prod_{i=k}^nJ_i\bI_{m_i,m_{i-1}}=\bI_{\infty,m_{k-1}}J'_{k},
\end{equation}
for some matrix $J'_k\in\cD_{m_{k-1}}$, and there exists a positive integer $N$ such that
\begin{equation}\label{sufficienteq1}
\prod_{i=k}^nJ_i\bI_{m_i,m_{i-1}}=\bI_{m_n,m_{k-1}}J'_{k}, \ \ \mbox{whenever}\ \ n>N.
\end{equation}
\end{lemma}

\begin{proof}
We first note that for all $i\ge k$, there holds
$$
J_{i+1}\bI_{m_{i+1},m_i}J_i\bI_{m_i,m_{i-1}}=\bI_{m_{i+1},m_{i-1}}J,
$$
where $J$ is the diagonal matrix in $\cD_{m_{i-1}}$ such that
$$
\supp J=\supp{J_{i+1}}\cap \supp J_i\cap\bN_{m_{i-1}}
$$
where
$$
\bN_k:=\{1,2,\dots, k\}.
$$
Consequently,
$$
\prod_{i=k}^nJ_i\bI_{m_i,m_{i-1}}=\bI_{m_{n},m_{k-1}}J
$$
where $J\in\cD_{m_{k-1}}$ satisfies
$$
\supp J=\biggl(\bigcap_{i=k}^n\supp J_i\biggr)\cap\bN_{m_{k-1}}.
$$
Since the set sequence $(\bigcap_{i=k}^n\supp J_i)\cap\bN_{m_{k-1}}$ is reverse nested as $n$ is increasing and is contained in the finite set $\bN_{m_{k-1}}$, that is,
$$
\biggl(\bigcap_{i=k}^{n+1}\supp J_i\biggr)\cap\bN_{m_{k-1}}\subseteq\biggl(\bigcap_{i=k}^n\supp J_i\biggr)\cap\bN_{m_{k-1}}\subseteq \bN_{m_{k-1}},
$$
we conclude that
$$
\lim_{n\to\infty}\biggl(\bigcap_{i=k}^n\supp J_i\biggr)\cap\bN_{m_{k-1}}
$$
exists and there exists some $N\in\bN$ such that for all $n>N$
$$
\biggl(\bigcap_{i=k}^n\supp J_i\biggr)\cap\bN_{m_{k-1}}=\lim_{n\to\infty}\biggl(\bigcap_{i=k}^n\supp J_i\biggr)\cap\bN_{m_{k-1}}.
$$
Let $J'_k$ be the diagonal matrix in $\cD_{m_{k-1}}$ such that
$$
\supp J'_k=\lim_{n\to\infty}\biggl(\bigcap_{i=k}^n\supp J_i\biggr)\cap\bN_{m_{k-1}}.
$$
This implies the validity of \eqref{limit00}.
Moreover, we get that for $n>N$
$$
\prod_{i=k}^nJ_i\bI_{m_i,m_{i-1}}=\bI_{m_{n},m_{k-1}}J'_k,
$$
which completes the proof.
\end{proof}

We are now ready to prove the main result of the section. The following result extends Theorem 4.3 of \cite{XuZhang} from square matrices of a fixed size to nonsquare matrices of increasing sizes.

\begin{theorem}\label{sufficientthm}
Let $\bW:=\{\bW_n\}_{n=1}^\infty$ with $\bW_n\in \bR^{m_n\times m_{n-1}}$. If
\begin{equation}\label{sufficient2}
\bW_n=\bI_{m_n,m_{n-1}}+\bP_n\mbox{ and }\sum_{n=1}^\infty\|\bP_n\|_p<+\infty,
\end{equation}
then the infinite product (\ref{infinitematrixproducts}) converges for all $J_n\in\cD_{m_n}$, $n\in\bN$.
\end{theorem}
\begin{proof} Since $\cB(\bR^d,\ell^p)$ is a Banach space, it suffices to show that $\{\bI_{\infty,m_{n}}\prod_{k=1}^n J_k \bW_k\}_{n=1}^\infty$ is a Cauchy sequence in $\cB(\bR^d,\ell^p)$. First notice that for $n'>n$,
\begin{equation}\label{sufficientthmeqs}
    \biggl\|\bI_{\infty,m_n}\prod_{i=1}^n J_i\bW_i-\bI_{\infty,m_{n'}}\prod_{i=1}^{n'} J_i\bW_i\biggr\|_p=\biggl\|\bI_{m_{n'},m_n}\prod_{i=1}^n J_i\bW_i-\prod_{i=1}^{n'} J_i\bW_i\biggr\|_p.
\end{equation}

By the second inequality of \eqref{sufficient2}, for a given  $\varepsilon>0$, there exists $q\in\bN$ such that
\begin{equation}\label{sufficientthmeq1}
\sum_{j=q+1}^\infty \|\bP_j\|_p<\varepsilon.
\end{equation}
By Lemma \ref{sufficientlemma}, when $n'>n$ are large enough, there exists $J'_k\in \cD_{m_{k-1}}$, $1\le k\le q+1$, such that
\begin{equation}\label{sufficientthmeq2}
\prod_{i=k}^nJ_i\bI_{m_i,m_{i-1}}=\bI_{m_n,m_{k-1}}J'_k\ \ \mbox{and}\ \ \prod_{i=k}^{n'}\bI_{m_i,m_{i-1}}=\bI_{m_{n'},m_{k-1}}J'_k,\ \ 1\le k\le q+1.
\end{equation}

With these preparations, we compute that
$$
\prod_{i=1}^n J_i\bW_i=\prod_{i=1}^n (J_i\bI_{m_i,m_{i-1}}+J_i\bP_i).
$$
Expanding the product on the right hand side of the above equation yields
\begin{equation}\label{sufficienteq3}
\prod_{i=1}^n J_i\bW_i=\prod_{i=1}^{n} J_i\bI_{m_i,m_{i-1}}
+\sum_{l=1}^{n}\sum_{1\le j_1<j_2<\cdots<j_l\le n}\biggl(\prod_{i=j_l+1}^nJ_i\bI_{m_i,m_{i-1}}\biggr)\biggl(\prod_{i=1}^l\biggl(J_{j_i}\bP_{j_i}\prod_{k=j_{i-1}+1}^{j_i-1}J_k\bI_{m_k,m_{k-1}}\biggr)\biggr),
\end{equation}
where $j_0:=0$. A key observation from (\ref{sufficientthmeq2}) is that the terms in
$$
\bI_{m_{n'},m_n}\sum_{1\le j_1<j_2<\cdots<j_l\le n}\left(\prod_{i=j_l+1}^nJ_i\bI_{m_i,m_{i-1}}\right)\left(\prod_{i=1}^l\left(J_{j_i}\bP_{j_i}\prod_{k=j_{i-1}+1}^{j_i-1}J_k\bI_{m_k,m_{k-1}}\right)\right),
$$
which appear in (\ref{sufficienteq3}) and those in
$$
\sum_{1\le j_1<j_2<\cdots<j_l\le n'}\left(\prod_{i=j_l+1}^{n'}J_i\bI_{m_i,m_{i-1}}\right)\left(\prod_{i=1}^l\left(J_{j_i}\bP_{j_i}\prod_{k=j_{i-1}+1}^{j_i-1}J_k\bI_{m_k,m_{k-1}}\right)\right)
$$
are identical if $j_l\le q$.

We now examine the difference
$$
\Delta:=\bI_{m_{n'},m_n}\prod_{i=1}^n J_i\bW_i-\prod_{i=1}^{n'} J_i\bW_i.
$$
Use equation (\ref{sufficienteq3}) with the fact pointed out above so that the identical terms appearing in the two products are canceled. Employing properties (\ref{matrixcon1}), (\ref{matrixcon2}), and (\ref{matrixcon3}) of the matrix norm, we obtain
$$
\begin{aligned}
\|\Delta\|_p\le &\sum_{j=q+1}^n \|\bP_j\|_p+\sum_{j=q+1}^{n'} \|\bP_j\|_p+\sum_{l=2}^{n}\sum_{\substack{1\le j_1<j_2<\cdots<j_l\le n\\j_l>q}}\prod_{k=1}^l\|\bP_{j_k}\|_p\\
&+\sum_{l=2}^{n}\sum_{\substack{1\le j_1<j_2<\cdots<j_l\le n'\\j_l>q}}\prod_{k=1}^l\|\bP_{j_k}\|_p+\sum_{l=n+1}^{n'}\sum_{1\le j_1<j_2<\cdots<j_l\le n'}\prod_{k=1}^l\|\bP_{j_k}\|_p.
\end{aligned}
$$
By inequality (\ref{sufficientlemma1}) in Lemma \ref{sufficientlemma}, the inequality above and (\ref{sufficientthmeq1}), we have for large enough integers $n'>n$ that
$$
\|\Delta\|_p\le 2 \biggl(\sum_{i=q+1}^{\infty}\|\bP_i\|_p\biggr)\exp\biggl(\sum_{i=1}^\infty\|\bP_i\|_p\biggr)\le 2\varepsilon\exp\left(\sum_{i=1}^\infty\|\bP_i\|_p\right).
$$
According to equation (\ref{sufficientthmeqs}), we confirm that $\{\bI_{\infty,m_{n}}\prod_{k=1}^n J_k\bW_k\}_{n=1}^\infty$ is a Cauchy sequence in $\cB(\bR^d,\ell^p)$, which completes the proof.
\end{proof}

By choosing $J_n=\bI_{m_n}$ for every $n\in\bN$ in (\ref{limit1}) and \eqref{infinitematrixproducts}, we obtain from Theorem \ref{sufficientthm} a sufficient condition for convergence of infinite products of matrices with increasing sizes. This result can be regarded as a natural generalization of the classical result for square matrices in (\cite{Wedderburn}, page 127) as mentioned above. We now present this result as a corollary of Theorem  \ref{sufficientthm}.

\begin{corollary}\label{classicalsufficient}
Let $\bW:=\{\bW_n\}_{n=1}^\infty$ with $\bW_n\in \bR^{m_n\times m_{n-1}}$. If
\begin{equation}\label{sufficient}
\bW_n=\bI_{m_n,m_{n-1}}+\bP_n\mbox{ and }\sum_{n=1}^\infty\|\bP_n\|_p<+\infty,
\end{equation}
then
$$
\prod_{n=1}^\infty \bW_n:=\lim_{n\to\infty}\bI_{\infty,m_n}\prod_{k=1}^n\bW_k
$$
converges in $\cB(\bR^d,\ell^p)$.
\end{corollary}

\section{Convergence of Deep Neural Networks with Increasing Widths}\label{ConvergenceDNN}
\setcounter{equation}{0}

Based on the results in Section 3, we shall investigate the convergence of a general ReLU network with increasing widths and then specify to the important convolutional neural networks. To this end, we first establish a useful sufficient condition guaranteeing the existence of limit (\ref{limit2}).

\begin{theorem}\label{theoremonlimit2}
Let $\bW:=\{\bW_n\}_{n=1}^\infty$ with $\bW_n\in \bR^{m_n\times m_{n-1}}$, and $\bbb:=\{\bbb_n\}_{n=1}^\infty$ with $\bbb_n\in\bR^{m_n}$ be a sequence of weight matrices and bias vectors, respectively. If
\begin{equation}\label{limit2sufficient3}
\sum_{n=1}^\infty \|\bbb_n\|_p<+\infty,
\end{equation}
\begin{equation}\label{limit2sufficient1}
\prod_{k=i}^\infty J_k\bW_k:=\lim_{n\to\infty}\bI_{\infty,m_n}\prod_{k=i}^n J_k\bW_k\mbox{ converges for every }i\ge1,
\end{equation}
and there exists a positive constant $C$ such that
\begin{equation}\label{limit2sufficient2}
\prod_{j=i}^n \|\bW_j\|_p\le C\mbox{ for all }1\le i\le n<+\infty,
\end{equation}
then the limit (\ref{limit2}) exists.
\end{theorem}
\begin{proof}
It suffices to show that
$$
\bc_n:=\bI_{\infty,m_n}\sum_{i=1}^n\biggl(\prod_{k=i+1}^n J_k\bW_k\biggr)J_i\bbb_i, \ \ n\in\bN
$$
forms a Cauchy sequence in $\ell^p$. Let $\varepsilon>0$ be arbitrary. By condition (\ref{limit2sufficient3}), there exists some $q\in\bN$ such that
\begin{equation}\label{sumofbbbi}
    \sum_{i=q+1}^\infty \|\bbb_i\|_p<\varepsilon.
\end{equation}
By hypothesis (\ref{limit2sufficient1}), when $n, n'$ are big enough with $n'>n$, there holds for all $i=1,2, \dots, q$ that
\begin{equation}\label{convergenceCondition}
\biggl\|\bI_{\infty,m_{n'}}\prod_{k=i+1}^{n'}J_k\bW_k-\bI_{\infty,m_{n}}\prod_{k=i+1}^{n}J_k\bW_k\biggr\|_p\le \varepsilon.
\end{equation}
For such $n'>n>q$, we estimate $\|\bc_{n'}-\bc_{n}\|_p$. To this end, we let
$$
\bd_{n',n,q}:=\sum_{i=1}^q\biggl(\bI_{\infty,m_{n'}}\prod_{k=i+1}^{n'}J_k\bW_k-\bI_{\infty,m_{n}}\prod_{k=i+1}^{n}J_k\bW_k\biggr)J_i\bbb_i.
$$
By (\ref{convergenceCondition}), we have for big enough $n, n'$ with $n'>n$ that
\begin{equation}\label{dn'np}
\|\bd_{n',n,q}\|_p\le \sum_{i=1}^q\biggl\|\bI_{\infty,m_{n'}}\prod_{k=i+1}^{n'}J_k\bW_k-\bI_{\infty,m_{n}}\prod_{k=i+1}^{n}J_k\bW_k\biggr\|_p\|\bbb_i\|_p\le \varepsilon\sum_{i=1}^q\|\bbb_i\|_p.
\end{equation}
Note that
\begin{equation}\label{differencecn-cn'}
\bc_{n'}-\bc_{n}=\bd_{n',n,q}+\bI_{\infty,m_{n'}}\sum_{i=q+1}^{n'}\biggl(\prod_{k=i+1}^{n'}J_k\bW_k\biggr)J_i\bbb_i+\bI_{\infty,m_{n}}\sum_{i=q+1}^{n}\biggl(\prod_{k=i+1}^{n}J_k\bW_k\biggr)J_i\bbb_i.
\end{equation}
Using \eqref{differencecn-cn'} and employing \eqref{dn'np}, (\ref{matrixcon3}), (\ref{limit2sufficient2}), and \eqref{sumofbbbi}, we have for big enough $n'>n$ that
$$
\begin{aligned}
\|\bc_{n'}-\bc_{n}\|_p
&\le \|\bd_{n',n,q}\|_p+\sum_{i=q+1}^{n'}\biggl(\prod_{k=i+1}^{n'}\|\bW_k\|_p\biggr)\|\bbb_i\|_p+\sum_{i=q+1}^{n}\biggl(\prod_{k=i+1}^{n}\|\bW_k\|_p\biggr)\|\bbb_i\|_p\\
&\le \varepsilon\sum_{i=1}^q\|\bbb_i\|_p+2C\sum_{i=q+1}^{\infty}\|\bbb_i\|_p\\
&\le \varepsilon\biggl(\sum_{i=1}^q\|\bbb_i\|_p+2C\biggr).
\end{aligned}
$$
This shows that $\bc_n$ is a Cauchy sequence in $\ell^p$ and thus it converges.
\end{proof}

We are now able to present sufficient conditions for pointwise convergence of deep ReLU networks with increasing widths.

\begin{theorem}\label{finalsufficient}
Let $\bW:=\{\bW_n\}_{n=1}^\infty$ with $\bW_n\in \bR^{m_n\times m_{n-1}}$, and $\bbb:=\{\bbb_n\}_{n=1}^\infty$ with $\bbb_n\in\bR^{m_n}$ be a sequence of weight matrices and bias vectors, respectively. If the weight matrices satisfy
\begin{equation}\label{finalsufficientcon1}
   \bW_n=\bI_{m_n,m_{n-1}}+\bP_n, \ \  n\ge 1, \ \ \sum_{n=1}^\infty\|\bP_n\|_p<+\infty
\end{equation}
and the bias vectors $\bbb_i$, $i\in\bN$, satisfy
$$
\sum_{n=1}^\infty \|\bbb_n\|_p<+\infty,
$$
then the ReLU neural networks $\cN_n$ converge pointwise on $[0,1]^d$.
\end{theorem}
\begin{proof}
By Theorem \ref{convergenceRELU}, it suffices to show under the given conditions of this theorem, limits  (\ref{limit1}) and  (\ref{limit2}) exist for all $J_n\in\cD_{m_n}$, $n\in\bN$. First, by Theorem \ref{sufficientthm}, condition \eqref{finalsufficientcon1} ensures that limit (\ref{limit1}) exists for all $J_n\in\cD_{m_n}$, $n\in\bN$. It remains to confirm the existence of limit (\ref{limit2}). By the proof of Theorem \ref{sufficientthm}, (\ref{limit2sufficient1}) is satisfied under condition (\ref{finalsufficientcon1}). We can also verify that
$$
\prod_{j=i}^n \|\bW_j\|_p\le\prod_{j=i}^n (1+\|\bP_j\|_p)\le \prod_{j=i}^n\exp(\|\bP_j\|_p)\le \exp\biggl(\sum_{j=1}^\infty\|\bP_j\|_p\biggr),\ \ 1\le i\le n<+\infty.
$$
Therefore, condition (\ref{limit2sufficient2}) is also satisfied. By Theorem \ref{theoremonlimit2}, limit (\ref{limit2}) exists for all $J_n\in\cD_{m_n}$, $n\in\bN$.

We conclude that by Theorem \ref{convergenceRELU}, the deep ReLU network $\cN_n$ with increasing widths converges pointwise on $[0,1]^d$ as $n$ tends to infinity.
\end{proof}

\section{Convergence of Deep Convolutional Neural Networks}
\setcounter{equation}{0}

We now return to the main topic of the paper, which is to establish the convergence of deep convolutional neural networks. Specifically, in this section we  apply the result on convergence of deep ReLU neural networks with increasing widths obtained in Section \ref{ConvergenceDNN} to CNNs because as we have indicated in Section \ref{CNNs}, CNNs in the matrix form are  special cases of DNNs.

In our development, we will need to estimate the matrix norm of a matrix induced by the $\ell^p$ vector norm. This will be done by employing the Riesz-Thorin interpolation theorem (see, \cite{Folland}, page 200). We now recall the Riesz-Thorin interpolation theorem that for any matrix $A$ and $p\in[1,+\infty]$
\begin{equation}\label{RieszThorin}
    \|A\|_p\le \|A\|_1^{\frac1p}\|A\|_\infty^{1-\frac1p}.
\end{equation}

We first apply Theorem \ref{finalsufficient} directly to obtain convergence of CNNs.

\begin{theorem}\label{cnnconvergence1}
Let $\bw^{(n)}:=(\bw^{(n)}_0,\bw^{(n)}_1,\dots,\bw^{(n)}_{s_n})$, $s_n\in\bN$, $n\in\bN$, be a sequence of filter masks,  $\bbb:=\{\bbb_n\}_{n=1}^\infty$ with $\bbb_n\in\bR^{m_n}$ be a sequence of bias vectors, where the widths $m_n$ are defined by (\ref{cnnwidth}). If the bias vectors $\bbb_i$, $i\in\bN$, satisfy
$$
\sum_{n=1}^\infty \|\bbb_n\|_p<+\infty,
$$
and the filter masks satisfy
\begin{equation}\label{cnnsufficient1conPlus}
\bw^{(n)}_0=1\ \  \mbox{for all}\ \ n\in\bN
\end{equation}
and
 \begin{equation}\label{cnnsufficient1con2}
   \sum_{n=1}^\infty \sum_{j=1}^{s_n}|\bw^{(n)}_j|<+\infty,
\end{equation}
then the deep ReLU convolutional neural networks (\ref{cnn1}) converge pointwise on $[0,1]^d$.
\end{theorem}
\begin{proof}
For each $n\in\bN$, let $\bW_n$ be the weight matrix associated with the filter mask $\bw^{(n)}$ defined by (\ref{cnnweightmatrix}). Since $\bw^{(n)}_0=1$ for all $n\in\bN$, we may make the decomposition
$$
{\bW}_n=\bI_{m_n,m_{n-1}}+\bP_n
$$
where
$$
(\bP_n)_{jk}:=\left\{
\begin{aligned}
0,&\quad j\le k\\
\bw^{(n)}_{j-k},&\quad j> k,
\end{aligned}
\quad 1\le j\le m_n,1\le k\le m_{n-1}.
\right.
$$
Recall that the matrix norms \cite{Wedderburn} of a matrix $A\in\bR^{m'\times m}$ satisfy
\begin{equation}\label{matrixnorm1inf}
\|A\|_p=\left\{
\begin{array}{cc}
     \displaystyle{\max_{1\le i\le m'}\sum_{j=1}^m |A_{ij}|},&  p=+\infty,\\
     \displaystyle{\max_{1\le j\le m}\sum_{i=1}^{m'} |A_{ij}|},&  p=1.
\end{array}
\right.
\end{equation}
One hence sees that the matrix norms of $\bP_n$ satisfy
$$
\|\bP_n\|_p\le \sum_{j=1}^{s_n}|\bw^{(n)}_j|\mbox{ for both }p=1\mbox{ and }p=+\infty.
$$
By applying  the Riesz-Thorin interpolation theorem (\ref{RieszThorin}), for all $1\le p\le +\infty$, we find that the $p$-matrix norm of $\bP_n$ satisfies
\begin{equation}\label{boundonPn1}
\|\bP_n\|_p\le \sum_{j=1}^{s_n}|\bw^{(n)}_j|.
\end{equation}
It follows from \eqref{boundonPn1} that
\begin{equation}\label{boundonPn1**}
\sum_{n=1}^\infty \|\bP_n\|_p\le \sum_{n=1}^\infty \sum_{j=1}^{s_n}|\bw^{(n)}_j|.
\end{equation}
Combining \eqref{boundonPn1**} with condition (\ref{cnnsufficient1con2}) leads to the following result
$$
\sum_{n=1}^\infty \|\bP_n\|_p<+\infty.
$$
The result of this theorem now follows directly from Theorem \ref{finalsufficient}.
\end{proof}

The hypothesis \eqref{cnnsufficient1conPlus} appearing in Theorem \ref{cnnconvergence1} imposes a rather restricted condition on the filter masks of the CNNs. Below, we weaken this condition to obtain a more general convergence result. For this purpose, instead of applying  Theorem \ref{finalsufficient} directly to the weight matrices defined by the filter masks of the CNNs, we appeal to  Theorem \ref{convergenceRELU}.

\begin{theorem}\label{cnnconvergence}
Let $\bw^{(n)}:=(\bw^{(n)}_0,\bw^{(n)}_1,\dots,\bw^{(n)}_{s_n})$, $s_n\in\bN$, $n\in\bN$, be a sequence of filter masks,  $\bbb:=\{\bbb_n\}_{n=1}^\infty$ with $\bbb_n\in\bR^{m_n}$ be a sequence of bias vectors, where the widths $m_n$ are defined by (\ref{cnnwidth}). If the bias vectors $\bbb_i$, $i\in\bN$, satisfy
$$
\sum_{n=1}^\infty \|\bbb_n\|_p<+\infty,
$$
and the filter masks satisfy
\begin{equation}\label{cnnsufficientcon1}
   \prod_{n=1}^\infty \bw^{(n)}_0\mbox{ converges to a nonzero limit}
   \end{equation}
   and
 \begin{equation}\label{cnnsufficientcon2}
   \sum_{n=1}^\infty \frac{\displaystyle{\sum_{j=1}^{s_n}|\bw^{(n)}_j|}}{|\bw^{(n)}_0|}<+\infty,
\end{equation}
then the deep ReLU convolutional neural networks (\ref{cnn1}) converge pointwise on $[0,1]^d$.
\end{theorem}
\begin{proof}
For each $n\in\bN$, let $\bW_n$ be the weight matrix associated with the filter mask $\bw^{(n)}$ defined by (\ref{cnnweightmatrix}).
By Theorem \ref{convergenceRELU}, we only need to show that under the given conditions of this theorem, limits  (\ref{limit1}) and  (\ref{limit2}) both exist for all $J_n\in\cD_{m_n}$, $n\in\bN$. By condition (\ref{cnnsufficientcon1}), we derive that $\bw^{(n)}_0\ne0$ for all $n\in\bN$. Set
$$
\tilde{\bW}_n:=\frac1{\bw^{(n)}_0}\bW_n,\ \ \mbox{for all}\ \ n\in\bN.
$$
Then we may factor $\tilde{\bW}_n$ as
$$
\tilde{\bW}_n=\bI_{m_n,m_{n-1}}+\bP_n,
$$
where
$$
(\bP_n)_{jk}:=\left\{
\begin{aligned}
0,&\quad j\le k\\
\frac{\bw^{(n)}_{j-k}}{\bw^{(n)}_0},&\quad j> k,
\end{aligned}
\quad 1\le j\le m_n,1\le k\le m_{n-1}.
\right.
$$
By (\ref{matrixnorm1inf}), we have for all $n\in\bN$ that
$$
\|\bP_n\|_p\le \frac{\displaystyle{\sum_{j=1}^{s_n}|\bw^{(n)}_j|}}{|\bw^{(n)}_0|}\mbox{ for both }p=1\mbox{ and }p=+\infty.
$$
Again, applying the Riesz-Thorin interpolation theorem (\ref{RieszThorin}), we then get for all $1\le p\le +\infty$ that
\begin{equation}\label{boundonPn}
\|\bP_n\|_p\le \frac{\displaystyle{\sum_{j=1}^{s_n}|\bw^{(n)}_j|}}{|\bw^{(n)}_0|}.
\end{equation}
As a result, by (\ref{cnnsufficientcon2}), the matrices $\tilde{W}_n$ satisfies condition (\ref{sufficient2}). In particular, we obtain that
\begin{equation}\label{p-norm_of_Pn}
\sum_{n=1}^\infty\|\bP_n\|_p<+\infty.
\end{equation}
By Theorem \ref{sufficientthm}, the infinite product of matrices
$$
\prod_{n=1}^\infty J_n\tilde{\bW}_n
$$
converges for all $J_n\in\cD_{m_n}$, $n\in\bN$. Noticing the definition of matrices $\tilde{\bW}_i$, we observe the following equation
$$
\prod_{i=1}^n J_i\bW_i=\biggl(\prod_{i=1}^n\bw^{(i)}_0\biggr)\prod_{i=1}^n J_i\tilde{\bW}_i.
$$
We thus conclude from condition (\ref{cnnsufficientcon1}) that limit (\ref{limit1}) exists for all $J_n\in\cD_{m_n}$, $n\in\bN$.

We next apply Theorem \ref{theoremonlimit2} to prove the existence of limit (\ref{limit2}) for all $J_n\in\cD_{m_n}$, $n\in\bN$. Condition (\ref{limit2sufficient3}) is automatically satisfied by the assumption of this theorem. Condition (\ref{limit2sufficient1}) also holds true by similar arguments as those developed above. It remands to verify condition (\ref{limit2sufficient2}). To this end, we fix  $n\in\bN$ and observe for all $i\le n$ that
$$
\begin{aligned}
\prod_{j=i}^n \|\bW_j\|_p&=\biggl(\prod_{j=i}^n|\bw^{(j)}_0|\biggr)\prod_{j=i}^n \|\tilde{\bW}_j\|_p\\
&=\biggl(\prod_{j=i}^n|\bw^{(j)}_0|\biggr)\prod_{j=i}^n \|\bI_{m_n,m_{n-1}}+\bP_n\|_p.
\end{aligned}
$$
It follows from the equation above that
\begin{equation}\label{Good-inquality}
\prod_{j=i}^n \|\bW_j\|_p\le \biggl(\prod_{j=i}^n|\bw^{(j)}_0|\biggr)\exp\biggl(\sum_{j=1}^\infty\|\bP_j\|_p\biggr).
\end{equation}
Note that condition (\ref{cnnsufficientcon1}) ensures that there exists a positive constant $C$ such that
 \begin{equation}\label{cnnsufficientcon3}
   \prod_{j=i}^n |\bw^{(j)}_0|\le C,\ \ \mbox{for all}\ \ i,n \ \ \mbox{with}\ \ 1\le i\le n<+\infty.
   \end{equation}
By substituting \eqref{p-norm_of_Pn} and (\ref{cnnsufficientcon3}) into the right hand side of inequality \eqref{Good-inquality}, condition (\ref{limit2sufficient2}) is fulfilled. Theorem \ref{theoremonlimit2} thus ensures that limit (\ref{limit2}) also exists for all $J_n\in\cD_{m_n}$, $n\in\bN$. That is, the hypothesis of Theorem \ref{convergenceRELU} is satisfied, and consequently, the deep ReLU convolutional neural networks (\ref{cnn1}) converge pointwise on $[0,1]^d$.
\end{proof}

Theorem \ref{cnnconvergence} extends Theorem \ref{cnnconvergence1} by replacing the restrictive hypothesis \eqref{cnnsufficient1conPlus} with significantly weaker condition \eqref{cnnsufficientcon1}.

A comment on condition (\ref{cnnsufficientcon1}) is in order.
A classical result on infinite products of scalars (see, \cite{Stein}, pages 246-247) ensures that if
$$
\bw^{(n)}_0=1+\lambda_n,\ |\lambda_n|\le \delta<1,\ \mbox{ and }\sum_{n=1}^\infty |\lambda_n|<+\infty
$$
then condition (\ref{cnnsufficientcon1}) holds true.

A final remark is that if the filter masks have a fixed length $s$ as they do in many convolutional neural networks, then condition (\ref{cnnsufficientcon2}) can be simplified as
$$
 \sum_{n=1}^\infty \frac{|\bw^{(n)}_j|}{|\bw^{(n)}_0|}<+\infty,\mbox{ for every }1\le j\le s.
$$
A necessary condition for the above to hold is that
$$
\lim_{n\to \infty}\frac{|\bw^{(n)}_j|}{|\bw^{(n)}_0|}=0, \ \ \mbox{for all}\ \ 1\le j\le s.
$$
This condition provides a guideline for the choice of the filter masks for CNNs.

{\small
\bibliographystyle{amsplain}

}

\end{document}